\documentclass[3p,11pt,nopreprintline]{elsarticle}

\usepackage{graphicx} 
\usepackage{amsmath}
\usepackage{amssymb}
\usepackage{amsthm}
\usepackage{multirow}
\usepackage[dvipsnames]{xcolor}
\usepackage{enumitem}
\usepackage{makecell}
\usepackage{hyperref}        
\usepackage{cleveref}
\usepackage{pdfpages}
\usepackage{bm}

\usepackage{pifont}
\newcommand{\cmark}{{\color{Green}\ding{51}}}%
\newcommand{\pmark}{{\color{Gray}$\boldsymbol{\thicksim}$}}%
\newcommand{\xmark}{{\color{Red}\ding{55}}}%
\newcommand{\nmark}{{\scriptsize\color{gray} N/A}}

\usepackage{booktabs} %

\newtheorem{proposition}{Proposition}

\newcommand{\bx}{\boldsymbol{x}}

\newcommand{\bd}{\bm{\mathrm{d}}}

\newcommand{\balpha}{\boldsymbol{\alpha}}

\begin{document}

\begin{frontmatter}

\title{Revisiting Explainable AI through Model-Independent Concept Dictionaries}

\author[toronto,vector]{Thomas Schnake\corref{cor1}}
\ead{t.schnake@utoronto.ca}
\author[ikim,bifold]{Doreen Schöppenthau}
\author[ikim,bifold]{Alexander Meyer}
\author[laval,mila]{Jacques Corbeil}
\author[bifold,tub,korea,mpi]{Klaus-Robert Müller\corref{cor1}}
\ead{klaus-robert.mueller@tu-berlin.de}
\author[bifold,ikim]{Grégoire Montavon\corref{cor1}}
\ead{gregoire.montavon@charite.de}

\cortext[cor1]{Corresponding author}

\address[bifold]{BIFOLD\,--\,Berlin Institute for the Foundations of Learning and Data, Berlin, Germany}
\address[tub]{Machine Learning Group, Technische Universit\"at Berlin, Germany}
\address[ikim]{Institute for AI in Medicine (IKIM), Charité -- Universitätsmedizin Berlin, Germany}
\address[toronto]{Department of Chemistry, Chemical Physics Theory Group, University of Toronto, Toronto, Canada}
\address[vector]{Vector Institute for Artificial Intelligence, Toronto, Canada}
\address[korea]{Department of Artificial Intelligence, Korea University, Seoul, Korea}
\address[laval]{Department of molecular medicine, Université Laval, Québec, QC, Canada}
\address[mila]{Mila\,--\,Quebec Artificial Intelligence
Institute, Canada}
\address[mpi]{Max Planck Institute for Informatics, Saarbr\"ucken, Germany}

\begin{abstract}
Modern applications of AI rely on increasingly complex models. Explainable AI (XAI) has emerged as a set of techniques aimed at improving model transparency. However, existing XAI methods typically assume input features to be inherently interpretable, or they rely on intermediate internal abstractions that are difficult to characterize and highly architecture-specific, hindering consistent use across models. To address these limitations, we propose DictXAI, a method that defines concepts directly in the input domain via a \textit{dictionary}---a large, potentially overcomplete set of predefined elements, each carrying an interpretable meaning. Technically, DictXAI first computes a sparse code of the input and then attributes the model's prediction to the associated dictionary elements. We demonstrate the actionable nature of DictXAI explanations, showing that they can attribute AI malfunctions (e.g.,~Clever Hans effects) directly to identifiable artifact patterns in the data, while fostering human-AI alignment on intricate biomedical signals. We further demonstrate our method's ability to operate across a wide variety of dictionaries, including learned image bases, analytically defined waveforms for electrocardiography, and experimentally acquired dictionary elements. Overall, our results show that DictXAI provides more interpretable, actionable, and architecture-agnostic insights than classical XAI or existing concept-based approaches.
\end{abstract}


\end{frontmatter}

\section{Introduction}

Recent advances in artificial intelligence (AI) have produced models of unprecedented scale and complexity. Explainable AI (XAI)~\cite{samek2021xai_review, Arrieta2020} addresses the black-box nature of these models, aiming to reintroduce transparency into their prediction strategies---whether to meet regulatory demands~\cite{Zhang2022} or to guide scientific inquiry~\cite{Wong2024XaiAntibiotics, esders2025xai_qc}. Although XAI methods can scale to such models, rigorously assessing explanation quality remains a multifaceted challenge~\cite{swartout1993desiderata, Nauta2023}. Explanations should not only be faithful to the model's decision strategy, but also be presented in a format that is understandable to the explainee and actionable for relevant downstream tasks.

\medskip

In pursuit of more human-aligned explanations, concept-based XAI~(e.g.,~\cite{kim2018TCAV, achtibat2023concept_relevance}) has emerged as a prominent direction. Its objective is to explain model behavior via high-level semantic units (referred to as ``concepts'') rather than raw input features. In vision tasks, for example, concepts may include textures or object parts, each abstracting over diverse pixel arrangements. To identify such concepts, dominant paradigms typically probe internal representations by isolating neurons or activation subspaces that respond to specific inputs~\cite{achtibat2023concept_relevance, huben2024sparse}. Yet this reliance on intermediate layers faces fundamental limitations: neurons are frequently polysemantic~\cite{Dreyer2024}, and representations are typically not comparable across architectures.

\medskip

\begin{figure}[t]
    \makebox[\textwidth][c]{\includegraphics[width=1.05\textwidth]{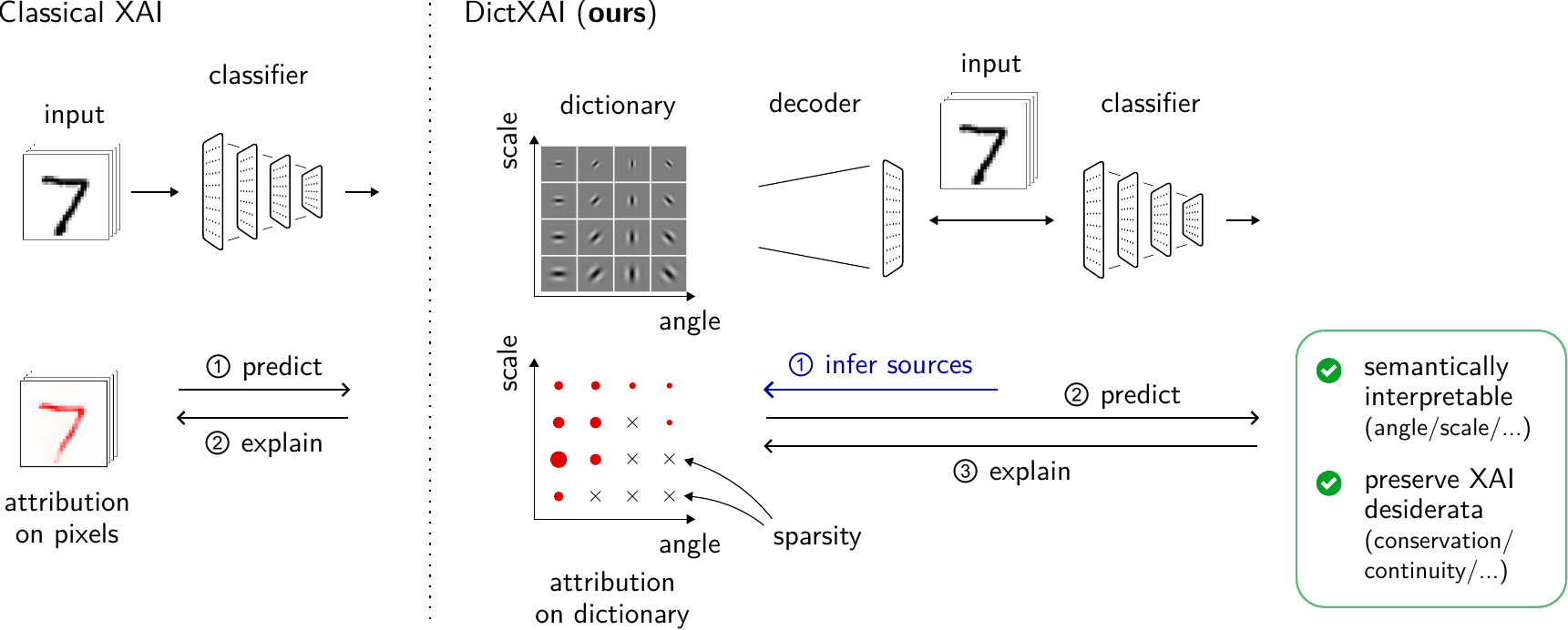}}
    \caption{\textbf{Overview of DictXAI.}
    \textit{Left (Classical XAI):} Standard workflows perform prediction and subsequently explain model decisions by computing attribution heatmaps directly on raw input pixels. \textit{Right (DictXAI, ours):} Our approach introduces a predefined dictionary of concepts (here, multi-scale oriented Gabor filters). First, sources are inferred via sparse coding of the input~\ding{172}. The input is then reconstructed via the decoder and passed to the classifier to predict~\ding{173}, after which explanations are propagated back through the classifier and the decoder to the dictionary coefficients~\ding{174}. This produces sparse explanations (where crosses denote zero attribution due to sparsity and red dot sizes indicate attribution magnitude) that are semantically interpretable along explicit axes (e.g., angle and scale) while preserving foundational XAI desiderata such as conservation and continuity.}
    \label{fig:overview_method}
\end{figure}

\medskip

To overcome these limitations, we introduce DictXAI, an explanation framework that represents concepts directly in the input domain by projecting inputs onto a predefined, human-inspectable dictionary (see Figure~\ref{fig:overview_method}). Depending on the application, dictionary elements may comprise mathematically parameterized primitives (e.g., Gabor filters or temporal wavelets with tunable scale and orientation), learned visual bases~\cite{Lee2006}, and empirical reference spectra derived from laboratory measurements. While concept-based methods in latent space typically enforce orthogonality~\cite{Zhou2018, Chormai2024}, DictXAI introduces overcompleteness as a core design principle, enabling richly expressive concept extraction directly in the input domain. Furthermore, unlike recent autoencoder-based approaches~\cite{huben2024sparse}, our framework avoids complex non-linear encoders; explanations are governed strictly by the original model and the explicit dictionary elements, preserving end-to-end transparency. Owing to this architectural simplicity, DictXAI naturally integrates with existing attribution techniques---such as Layer-wise Relevance Propagation (LRP)~\cite{bach2015lrp}, Shapley values~\cite{Strumbelj2010}, and higher-order extensions~\cite{eberle2022bilrp}. DictXAI thus inherits the formal faithfulness and computational flexibility of these established methods while substantially elevating their interpretability and human alignment.

\medskip

Through an application-grounded evaluation spanning multiple use cases, we demonstrate DictXAI's advantages over classical XAI methods and existing concept-based extensions. First, we show its utility in purging a model of its reliance on spurious correlations, the so-called Clever Hans effect~\cite{lapuschkin2019clever_hans, Kauffmann2025}. By selectively removing the offending input-space dictionary atoms, DictXAI blocks the flawed strategy at its source and prevents it from re-emerging through alternative internal circuits. Furthermore, we demonstrate DictXAI's ability to operate across a wide variety of application-relevant data modalities and concept spaces. This includes using dictionaries of amplitude-modulated sinusoidal (AMS) waveforms to interpret QRS complex properties (e.g., duration) predicted from electrocardiogram (ECG) signals, as well as utilizing spectral profiles of microbial isolates as an interpretable basis for mass spectrometry analysis. In each case, DictXAI derives human-understandable and actionable insights that go significantly beyond the capabilities of classical pixel-level attribution or existing concept-based methods.

\medskip

The remainder of this paper is structured as follows. Section~\ref{section:related_work} reviews related work in latent- and input-level concept-based XAI. Section~\ref{section:dictxai} introduces the mathematical framework of DictXAI, detailing its technical properties and the flexibility afforded by diverse dictionary choices. Section~\ref{section:evaluation} presents an application-grounded evaluation of our method. Finally, Section~\ref{section:conclusion} concludes the paper and outlines directions for future research.

\section{Related Work}
\label{section:related_work}

DictXAI formulates concept-based explanations directly in the input domain using overcomplete dictionaries. Accordingly, our work connects to two primary bodies of literature: concept-based XAI methods operating on internal representations (Section~\ref{section:relatedwork-latent}) and attribution methods that construct bases directly in the input domain (Section~\ref{section:relatedwork-input}). For a broader overview of XAI, including classical attribution techniques such as Layer-wise Relevance Propagation (LRP) upon which our explanation framework builds, we refer to comprehensive review papers~\cite{samek2021xai_review, Arrieta2020}.

\subsection{Concepts in Latent Activation Layers}
\label{section:relatedwork-latent}

The intermediate layers of deep neural networks provide a rich representational space in which human-understandable concepts can emerge~\cite{Yamins2014, Bau2020}. Numerous works in XAI have sought to leverage these internal representations. For instance, Concept Relevance Propagation (CRP)~\cite{achtibat2023concept_relevance} extends the LRP framework by filtering relevance flows through individual internal units (e.g., neurons or convolutional channels), where each unit is associated with a concept post hoc by inspecting the instances that maximally activate it. Moving from single units to layer-wide representations, \citet{kim2018TCAV} and \citet{Zhou2018} utilize auxiliary concept datasets to identify concept-discriminative directions---termed Concept Activation Vectors (CAVs)---enabling attribution across distributed semantic directions rather than isolated channels.

\smallskip

Subsequent work has generalized concept vectors to multidimensional subspaces, introducing completeness guarantees to the resulting explanations~\cite{Vielhaben2023, Chormai2024}. These frameworks bypass supervised labeling by extracting concept subspaces in an unsupervised manner via clustering or relevance-driven subspace analysis. In a related vein, \citet{fel2023concept_extr} cast concept discovery as an unsupervised dictionary learning problem, applying non-negative matrix factorization (NMF) to hidden activations to uncover latent concept atoms. More recently, several works have turned to training sparse autoencoders (SAEs) on intermediate representations to isolate highly granular, monosemantic latent concepts~\cite{huben2024sparse, bussmann2025learning, zaigrajew2025interpreting}.

\smallskip

Whereas all of the aforementioned methods are principally designed to operate in hidden layers, DictXAI shifts the explanation strategy entirely to the input domain. Grounding concepts directly in the input space confers three distinct advantages: (i) it avoids inheriting the black-box opacity of intermediate representations or introducing auxiliary autoencoding models fitted on them; (ii) unlike intermediate layers, the input domain provides rich, explicit semantics---such as scale, angle, or physical concentration---that dictionary elements can directly map to; and (iii) input-defined concepts are inherently model-agnostic, enabling direct, unmediated comparisons of decision strategies across differing architectures.

\subsection{Concepts within the Input Domain}
\label{section:relatedwork-input}

While latent-space approaches extract abstract representations from hidden layers, a parallel line of research decomposes explanations directly within the input domain using structured mathematical bases. For instance, \citet{vielhaben2024inspect_layer} proposed DFT-LRP, which introduces an inspection layer that allows propagation-based explanation methods to express relevance in a Fourier basis, particularly suited for mapping to the frequency domain that characterizes audio applications. A similar Fourier-basis approach, but applied to two-dimensional signals, was used by \citet{Kauffmann2025} to explain image-based anomaly detection models; there, frequency-domain explanations were shown to uniquely reveal the model's sensitivity to intricate technical artifacts, such as the presence or absence of image antialiasing.

\smallskip

Other frameworks leverage localized spatial-frequency representations. \citet{kasmi2025WAM} proposed the Wavelet Attribution Method (WAM), which transforms the input into the wavelet domain to obtain localized spatial-frequency representations, enabling standard gradient-based explanation methods to assign importance scores to individual wavelet functions. Similarly, \citet{kolek2022CartoonX} proposed CartoonX (later extended in~\cite{kolek2023ShearletX}), which transforms inputs into the wavelet domain and learns a sparse perturbation mask over the wavelet coefficients via rate-distortion optimization.

\smallskip

Despite their utility, these existing input-domain frameworks remain fundamentally constrained by fixed, non-redundant or orthogonal basis sets where the number and functional form of basis elements are strictly tied to predefined mathematical transforms. Consequently, they lack the flexibility to incorporate empirical domain knowledge, learned dictionaries, or non-orthogonal primitives. In contrast, DictXAI affords complete freedom in the choice of representation: dictionaries can comprise parameterized functions, data-driven learned atoms, or experimentally acquired physical spectra. Furthermore, by embracing overcompleteness, DictXAI can decouple subtle, highly correlated concept variations---such as continuous scale shifts in physiological waveforms or overlapping spectral signatures in microbiological mixtures---that rigid orthogonal transforms cannot isolate.

\section{The DictXAI Method}\label{section:dictxai}

In this section, we introduce DictXAI, a framework that explains machine learning predictions in terms of a predefined, potentially overcomplete dictionary of input-domain concepts. In the context of images, each dictionary element (or \textit{atom}) can represent, for example, an edge at a specific spatial location and orientation, or more complex pixel arrangements observed in the data. For a given input $\bx \in \mathbb{R}^d$ with prediction $y = f(\bx)$, DictXAI operates via three sequential steps:  (1)~express the input data in terms of the dictionary elements using linear sparse coding, (2)~feed the reconstructed input into the ML model to obtain the prediction, and (3)~attribute the prediction to dictionary atoms that contribute to it. This pipeline is illustrated at a high level in Figure~\ref{fig:overview_method}. We detail each step below.

\paragraph{Step 1: Linear Sparse Coding} DictXAI assumes a dictionary of size $K$, where each dictionary element $\bd_j \in \mathbb{R}^d$ is expressible in input space. Linear sparse coding sets as an objective to re-express the input data as a sparse combination of the dictionary elements \cite{Lee2006,Elad2010sparse_repres_book}. For a given data point $\bx \in \mathbb{R}^d$, this task can be formulated mathematically as the LASSO objective~\cite{Tibshirani1996}:
\begin{align}
\min_{\balpha} \;  \Big\{\,\frac12\,\Big\| \sum_{j=1}^K \alpha_j \bd_j - \bx\Big\|^2 + \lambda \sum_{j=1}^K |\alpha_j|\, \Big\}
\label{eq:sparsecoding}
\end{align}
where the variables $(\alpha_j)_j$ denote the weighting coefficients. The first term in the objective ensures the reconstruction error is minimized. The second term enforces sparsity by driving non-contributing coefficients $\alpha_j$ to zero. We note that, if these coefficients are calculated through an encoding function rather than optimized directly, the approach becomes conceptually equivalent to Sparse Autoencoders, commonly used in mechanistic interpretability~\cite{huben2024sparse}. However, directly solving Eq.~\eqref{eq:sparsecoding} offers two clear advantages. Because the coefficients $\alpha_j$ are optimized freely per sample, this optimization is not constrained by the capacity limits of a fixed encoder. Furthermore, the absence of an encoder eliminates structural architectural bias and avoids the risk of out-of-distribution generalization failure.

\paragraph{Step 2: Recovering the Forward Pass} Having extracted a sparse representation of the input data, we must now re-establish the exact input-output relationship to be explained. Standard linear sparse coding produces a lossy reconstruction; that is, the linear combination of dictionary elements $\hat{\bx} = \sum_{j=1}^K \alpha_j \bd_j$ does not exactly match the original input $\bx$. Feeding $\hat{\bx}$ into the classifier can alter the prediction of the downstream ML model as well as the decision strategy employed to reach it. To prevent this distortion and remain faithful to the original prediction behavior, we reformulate the sparse coding problem above to deliver an explicit, lossless input reconstruction:
\begin{align}
\min_{\balpha,\bd_0} \; \Big\{\,\frac12 \| \bd_0 \|^2 + \lambda \sum_{j=1}^K |\alpha_j|\,\Big\} \quad \text{s.t.} \quad \bx = \sum_{j=0}^K \alpha_j \bd_j \quad \text{and} \quad \alpha_0 = 1.
\label{eq:sparsecoding-rev}
\end{align}
In this formulation, we have introduced an instance-specific dictionary atom $\bd_0 \in \mathbb{R}^d$ whose magnitude acts as a slack variable that absorbs the residual reconstruction error. In practice, the solution to this optimization problem is equivalently obtained by solving Eq.~\eqref{eq:sparsecoding} for $\balpha$ and setting the residual atom to $\bd_0 = \bx - \hat{\bx}$.

\smallskip
We can now formally link the expressions of each atom in the dictionary to the model's prediction via the equation:
\begin{align}
y = f\Big(\sum_{j=0}^K \alpha_j \bd_j\Big).
\label{eq:forward}
\end{align}
In other words, our approach yields a formulation that is functionally equivalent to the original $y = f(\bx)$, but crucially, incorporates the interpretability structure uncovered in Step 1. The notion of improving interpretability while maintaining exact functional equivalence is much in the spirit of advanced interpretability techniques, such as virtual layers~\cite{vielhaben2024inspect_layer, Chormai2024}.

\paragraph{Step 3: Attribution to the Dictionary Atoms} Leveraging this computational structure, we can now focus on the problem of attributing the model prediction to the dictionary atoms. Many baseline frameworks commonly employed for input features can be readily extended to dictionary elements. We sketch the approach below for three cases representative of the major families of attribution methods, namely, occlusion-based, gradient-based, and propagation-based:
\begin{align}
\text{Occlusion~\cite{bluecher2022preddiff}} \quad & R_j = f(\bx) - f(\bx - \alpha_j \bd_j), \label{eq:dict_occlusion} \\
\text{Integrated Gradients (IG)~\cite{sundararajan2017integrated_gradient}} \quad & R_j = \int_0^1 \frac{\partial y}{\partial \alpha_j} \frac{\partial \alpha_j}{\partial t} \, dt, \label{eq:dict_ig} \\
\text{Layer-wise Relevance Propagation (LRP)~\cite{bach2015lrp}} \quad & R_j = \sum_{i=1}^d \frac{[\alpha_j\bd_j]_i}{x_i} R_i. \label{eq:dict_lrp}
\end{align}
For occlusion, removing atom $j$ corresponds directly to subtracting its contribution $\alpha_j \bd_j$ from the input. This formulation naturally generalizes to Shapley Value Sampling~\cite{Strumbelj2010}, which aggregates these marginal effects across atom coalitions. For IG, the integration path is given by linearly scaling the weighting coefficients from zero to their actual value, i.e., $\balpha(t) = t \balpha$. For LRP, input-feature attributions $(R_i)_{i=1}^d$ are propagated backward to the dictionary layer, here using the basic LRP-0 rule~\cite{Montavon2019} with the convention $0/0=0$. For improved stability, the dictionary layer and the model's first linear layer can be merged into a single linear layer, granting access to more robust propagation variants based on LRP-$\gamma$~\cite{Montavon2019}.

\subsection{Choice of Dictionary}
\label{sec:finding_dictionaries}

The choice of dictionary defines the conceptual basis for interpretability and directly governs the expressiveness of the resulting explanations. In practice, dictionaries fall into three primary categories, each presenting distinct trade-offs between flexibility, parametric transparency, and physical grounding.

\paragraph{Learned Dictionaries}
A fully data-driven approach constructs dictionaries using unsupervised learning algorithms such as sparse dictionary learning~\cite{olshausen1996emergence,Lee2006}. Most techniques initialize a random dictionary and alternate between identifying the optimal sparse code and updating the dictionary atoms to minimize reconstruction error; other approaches optimize end-to-end via gradient descent by incorporating a parameterized encoder function. To illustrate this behavior, \Cref{fig:dictionary_comparison}A displays a dictionary of $K=800$ atoms trained on the MNIST handwritten digit benchmark (for training details, see Supplementary Note~A). The resulting atoms capture refined, domain-specific primitives---here, an expressive variety of localized stroke contours. Because the dictionary atoms and sparse inference are jointly optimized for the target data distribution, learned dictionaries offer near-optimal trade-offs between sparsity and reconstruction fidelity. On the other hand, learned atoms lack explicit parametric formulas, which can complicate the qualitative interpretation of downstream explanations. Furthermore, data-driven dictionary learning assumes an abundant sample supply; data scarcity risks producing overfitted dictionary elements that impede the extraction of clear, human-aligned concepts.

\begin{figure}[t!]
    \centering
    \includegraphics[width=\textwidth]{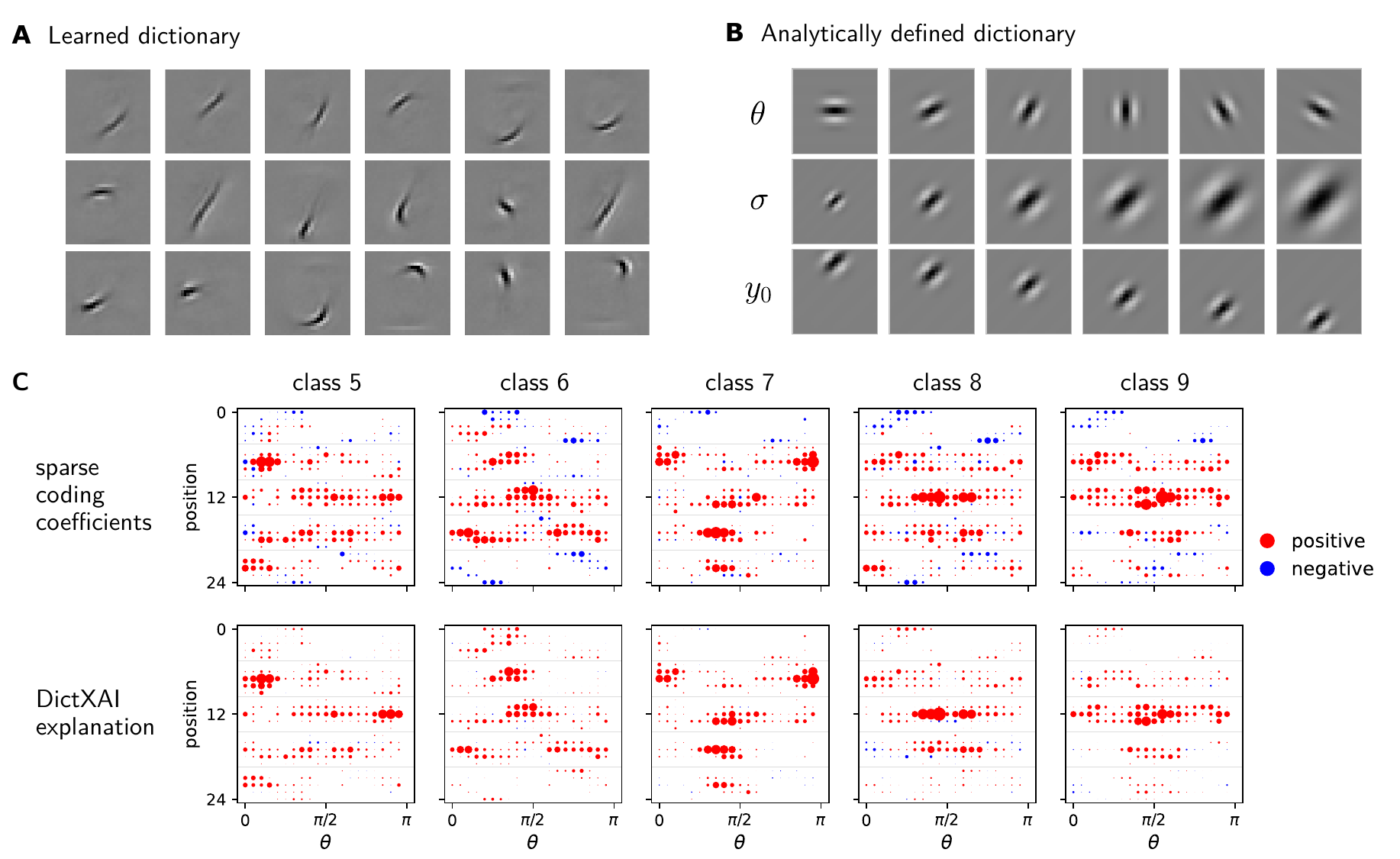}
    \caption{\textbf{Dictionary designs and concept attribution profiles on MNIST.}
    \textbf{A}~Subset of atoms of intermediate $L_2$ norm from a dictionary learned on the MNIST training set via sparse dictionary learning ($K = 800$; $\mathrm{MSE} = 0.017$ on $1\,000$ test images; $\lambda=10^{-3}$).
    \textbf{B}~Parameter sweeps from an analytically defined Gabor dictionary ($K\approx10\,000$; $\mathrm{MSE} = 0.032$ on the same $1\,000$ test images; $\lambda=10^{-3}$). Each row isolates variations along a single parameter while holding others fixed: orientation $\theta$ (top), envelope scale $\sigma$ (middle), and vertical center $y_0$ (bottom).
    \textbf{C}~Class-aggregated profiles for the Gabor dictionary across test digits (classes 5--9, 30 samples each). Panels display Hinton diagrams where marker size denotes magnitude and color denotes sign (red: positive, blue: negative). Top row: mean sparse coding coefficients; bottom row: DictXAI explanation (mean relevance scores). Values are plotted across the flattened $5\times 5$ spatial grid ($y$-axis, indices 0--24, row-major) and orientation $\theta \in [0, \pi]$ ($x$-axis), with the scale parameter $\sigma$ averaged out.}
    \label{fig:dictionary_comparison}
\end{figure}

\paragraph{Analytically Parameterized Dictionaries}
When data is scarce or domain knowledge suggests a natural structure for the relevant concepts, dictionaries can instead be defined in closed form and indexed by interpretable parameters. A prominent example is a family of Gabor filters (\Cref{fig:dictionary_comparison}B; formal description in Supplementary Note~A), which captures localized texture and frequency information across specific spatial frequencies, orientations, phase offsets, aspect ratios, and spatial locations. Discretizing these parameters yields a structured dictionary of analytical atoms. A major advantage of this design is that each atom $\bd_j$ possesses an explicit coordinate in the underlying parameter space. Consequently, downstream explanations can be mapped directly into the metric space spanned by these parameters, facilitating intuitive visualization. For example, when atoms vary along grid-discretized parameters such as orientation and spatial position, feature attributions can be rendered as a Hinton diagram over that parameter grid (\Cref{fig:dictionary_comparison}C). This enables practitioners to systematically identify which parameterized features the model relies on when making predictions, both locally and across the dataset. However, because analytical dictionaries are not adapted to the data distribution, spanning continuous parameter spaces requires extensive discretization: in our illustrative setting, an overcomplete dictionary comprising over $12\times$ more elements than the learned baseline still achieves almost two times higher reconstruction error ($\mathrm{MSE} = 0.032$ vs.\ $0.017$).

\paragraph{Dictionaries of Elementary Constituents}
A third category emerges in physical and biological domains where the observed signal is known to be an additive mixture of basic physical components in unknown proportions. This setting is ubiquitous in applications such as spectrometry and computational cytometry, where an aggregate measurement reflects a superposition of constituent profiles---such as a bulk chemical spectrum composed of pure molecular signatures, or bulk RNA sequencing deconvolved into cell-type-specific references. In these scenarios, the dictionary is neither learned statistically from the target dataset nor generated from an idealized functional form; rather, it comprises an empirical reference library isolated through prior experimental characterization. The principal advantage of such dictionaries is their direct physical grounding: each atom corresponds unambiguously to a concrete, real-world entity. However, their fidelity depends strictly on library completeness; unmodeled chemical contaminants or uncharacterized cell types cannot be captured within a fixed constituent basis.

\medskip

Importantly, these dictionary archetypes are not mutually exclusive. In principle, one could construct hybrid dictionaries that augment analytical or physical reference bases with learnable atoms. When fixed functional families or reference libraries cannot fully capture all factors of variation in the data, auxiliary learnable elements could absorb the remaining residual variance while preserving the explicit interpretability of the known components.

\subsection{Properties of DictXAI}

\newcommand{\pointlessparagraph}[1]{\medskip
\noindent
\emph{#1}}

In this section, we analyze DictXAI across multiple operational criteria, ranging from formally verifiable axiomatic guarantees to qualitative usability attributes. Together, these properties cover the five foundational explanation desiderata established by \citet{swartout1993desiderata} and adopted in modern XAI~\cite{Kauffmann2024,bender2026towards_desi}: faithfulness (D1), understandability (D2), sufficiency (D3), low overhead (D4), and runtime efficiency (D5). Table~\ref{table:technical} provides a systematic comparative summary alongside representative concept- and transform-based methods from the literature. To reflect the distinct dictionary paradigms established in \Cref{sec:finding_dictionaries}, we evaluate DictXAI under two primary operating modes: a \emph{curated} regime based on pre-defined analytical or physical components, and a fully data-driven \emph{learned} regime. As summarized in Table~\ref{table:technical}, while both regimes share identical faithfulness (D1) and computational complexity (D5), they exhibit complementary trade-offs across usability, steerability, and annotation requirements.

\pointlessparagraph{D1: Faithfulness (Conservation and Input Stability)} requires an explanation to reflect the model's true decision mechanism without distorting evidence or behaving erratically under minor noise. DictXAI satisfies this through two foundational guarantees.  First, \emph{conservation} ensures that decomposed relevance scores sum exactly to the model's output to permit unambiguous interpretation of relative evidence shares. By prepending a single linear dictionary layer to the network and executing standard attribution engines (e.g.,\ LRP or Shapley values) over this extended architecture, DictXAI directly inherits their conservation guarantees, satisfying $\sum_{j=0}^K R_j = \sum_{i=1}^d R_i = f(\bx)$. While transform propagation frameworks (DFT-LRP, WAM) also conserve relevance, methods rooted in directional projections (TCAV) or rate-distortion masking (CartoonX) are non-conservative by design. Second, \emph{input stability} demands that bounded input perturbations yield bounded variations in concept attributions. DictXAI guarantees this stability by lifting the smoothness of the underlying model directly into the dictionary domain:

\begin{proposition}[Lipschitz Continuity of Dictionary-Level Derivatives]
Recall from Eqs.~\eqref{eq:sparsecoding}--\eqref{eq:forward} that the model prediction can be expressed as
$
y = f(\bx) = f(\hat{\bx} + \bd_0),
$
where $\hat{\bx} = \sum_{j=1}^K \alpha_j \bd_j$ and the residual $\bd_0 = \bx - \hat{\bx}$ is held constant when differentiating with respect to $\balpha$ at $\bx$. If $\nabla f$ is $L$-Lipschitz continuous on a convex domain $\Omega \subseteq \mathbb{R}^d$, then for any dictionary atom $\bd_j$ ($j = 1, \dots, K$), the derivative map $\bx \mapsto \partial y / \partial \alpha_j$ is $(L \|\bd_j\|_2)$-Lipschitz continuous with respect to $\bx$ on $\Omega$.
\end{proposition}

\begin{proof}
Evaluating the partial derivative with respect to code coordinate $\alpha_j$ via the chain rule through the intermediate representation $\hat{\bx} = \sum_{j=1}^K \alpha_j \bd_j$ yields the directional derivative along atom $\bd_j$:
\[
\frac{\partial y}{\partial \alpha_j}(\bx) = \bd_j^\top \nabla f(\bx).
\]
For any two points $\bx_1, \bx_2 \in \Omega$, applying the Cauchy--Schwarz inequality followed by the $L$-Lipschitz continuity of $\nabla f$ gives:
\begin{align*}
\Big| \frac{\partial y}{\partial \alpha_j}(\bx_1) - \frac{\partial y}{\partial \alpha_j}(\bx_2)\Big| 
&= \big| \bd_j^\top \big(\nabla f(\bx_1) - \nabla f(\bx_2)\big) \big| \\[-2mm]
&\le \|\bd_j\|_2 \|\nabla f(\bx_1) - \nabla f(\bx_2)\|_2 \\
&\le L \|\bd_j\|_2 \|\bx_1 - \bx_2\|_2. \qedhere
\end{align*}
\end{proof}
\noindent Hence, input gradient stability strictly bounds variations in the dictionary domain, scaled by the atom norms $\|\bd_j\|_2$. Furthermore, under standard LASSO regularity conditions (such as the dictionary satisfying the general position assumption), the active coordinate support and its sign pattern remain locally invariant. On these support regions, the coordinate mappings $\bx \mapsto \alpha_j(\bx)$ are locally Lipschitz continuous. Because products of bounded Lipschitz maps remain Lipschitz, these stability guarantees extend directly to attribution methods based on coefficient-gradient products (such as Integrated Gradients and $\text{Gradient} \times \text{Input}$) and stabilized LRP propagation rules.

\begin{table}[t!]
\centering \small
\setlength{\tabcolsep}{4.5pt}
\makebox[\textwidth][c]{
\begin{tabular}{llcccccccc}\toprule
& &
CRP &
DRSA &
TCAV &
DFT-LRP &
WAM &
CartoonX &
\multicolumn{2}{c}{DictXAI (\textbf{ours})}\\
&  &
\cite{achtibat2023concept_relevance} &
\cite{Chormai2024} &
\cite{kim2018TCAV} &
\cite{vielhaben2024inspect_layer} &
\cite{kasmi2025WAM} &
\cite{kolek2022CartoonX} &\\[-3mm]\cmidrule{9-10}
&&&&&&&&
curated & learned\\\midrule
\multirow{2}{*}{D1}
& Conservation          & \pmark$^{\,(1)}$ & \cmark & \xmark & \cmark & \cmark & \xmark & \cmark & \cmark\\
& Input stability       & \cmark & \cmark & \nmark & \cmark & \cmark & \xmark & \cmark & \cmark\\\midrule
D2
& Concept steerability     & \xmark & \xmark & \cmark & \xmark & \xmark & \xmark & \cmark & \xmark\\\midrule
\multirow{2}{*}{D3}
& Concept granularity   & \cmark & \cmark & \cmark & \xmark$^{\,(2)}$ & \xmark$^{\,(2)}$ & \xmark$^{\,(2)}$ & \cmark & \cmark\\
& Concept discovery  & \cmark & \cmark & \xmark & \xmark & \xmark & \xmark & \xmark & \cmark\\\midrule
\multirow{2}{*}{D4}
& Low codebase overhead   & \pmark & \xmark & \xmark & \cmark & \cmark & \pmark & \cmark & \pmark\\
& Low annotation overhead & \pmark & \cmark & \xmark & \cmark & \cmark & \cmark & \pmark & \cmark\\\midrule
\multirow{2}{*}{D5}
& Compute cost (until concepts) & $f$   & $f$   & $f$   & $f$   & $f$   & $M f$ & \multicolumn{2}{c}{$f + \mathrm{SC}$}\\ 
& Compute cost (until pixels)   & $K f$ & $K f$ & \nmark & $f$ & $f$ & $M f$ & \multicolumn{2}{c}{$f + \mathrm{SC}$}\\ 
\bottomrule
\end{tabular}
}
\caption{Systematic comparison of concept-based XAI methods across technical properties, grouped by the five desiderata of \cite{swartout1993desiderata} (D1: faithfulness, D2: understandability, D3: sufficiency, D4: low overhead, D5: runtime efficiency). Symbols \cmark, \pmark, and \xmark\ denote satisfied, partially satisfied, and unsatisfied properties, respectively, while \nmark\ indicates that a property is not applicable due to the underlying attribution paradigm. 
DictXAI is evaluated under both domain-curated and data-driven learned dictionary regimes, assuming LRP as the base attribution engine. In D5, $f$ denotes the cost of a single model evaluation (forward and backward pass), $K$ is the number of queried concepts, $M$ is the number of optimization steps in CartoonX, and $\mathrm{SC}$ is the cost of solving the linear sparse coding subproblem. 
$^{(1)}$~Inherits conservation from LRP but strong dissipation due to single-channel relevance filtration. 
$^{(2)}$~Fundamentally limited by input space orthogonality.}
\label{table:technical}
\end{table}

\pointlessparagraph{D2: Understandability (Concept Steerability)} requires explanations to align with the language of domain experts, enabling the underlying concept vocabulary to be steered toward task-specific ontologies. DictXAI natively fulfills this property by allowing practitioners to directly inject domain knowledge into the dictionary---whether via expert-defined analytical bases, flexible overcomplete unions of them, or experimentally acquired prototypes. In contrast, rigid transform baselines (DFT-LRP, WAM, CartoonX) remain bound to a single, mathematically immutable transform, while unsupervised latent decompositions (DRSA) cannot be steered, relying on the tenuous assumption that unsupervised data statistics happen to align with the expert's semantic domain.

\pointlessparagraph{D3: Sufficiency (Granularity and Discovery)} demands that an explanation vocabulary possesses both the resolution and the scope needed to fully characterize model decisions. DictXAI addresses this through two complementary facets. First, \emph{concept granularity} governs the spatial and semantic resolution at which an explainer isolates evidence. While classical transforms remain constrained by input-space orthogonality, overcomplete dictionaries allow DictXAI to disentangle localized, highly fine-grained primitives directly from the input signal without collapsing distinct features into coarse averages. Second, \emph{concept discovery} ensures comprehensive coverage across the data manifold. Through learned dictionary elements, DictXAI accounts for hard-to-specify yet human-aligned concepts (e.g.,\ image parts or textures) that are missing from classical analytical families.

\pointlessparagraph{D4: Low Overhead (Codebase and Annotation Burden)} is essential for practical adoption. DictXAI balances these dimensions across its two operating modes. First, \emph{codebase overhead} remains minimal: prepending an input-level dictionary avoids intercepting intermediate activations, altering network architectures, or managing layer-specific backward hooks required by latent-space frameworks like CRP or DRSA. When using curated dictionaries, no auxiliary training pipelines are needed (\cmark); learned dictionaries introduce only a standard unsupervised dictionary learning pre-routine (\pmark). Second, the \emph{annotation overhead} is virtually absent. Unlike supervised approaches (e.g.,\ TCAV) that require gathering positive and negative concept exemplars, learned DictXAI requires solely unlabeled data (\cmark), while curated DictXAI requires only the initial parametric specification or assembly of a reference basis (\pmark).

\pointlessparagraph{D5: Runtime Efficiency (Computational Complexity)} enables deployment of explanation methods at scale (e.g.,\ on large models or on entire datasets). Letting $f$ denote the cost of a model forward--backward pass, intermediate-layer methods (CRP, DRSA) cost $f$ to attribute relevance to concepts but require $K$ sequential backward sweeps ($K f$) to map relevance from each queried concept back to the input, while perturbation techniques like CartoonX incur $M$ iterative optimization steps ($M f$). In contrast, DictXAI performs sparse coding once per sample at cost $\mathrm{SC}$ and executes a single forward--backward pass on the dictionary-extended model ($f + \mathrm{SC}$). This achieves simultaneous concept attribution and pixel-level grounding without scaling linearly with the number of queried concepts $K$.

\section{Evaluation}
\label{section:evaluation}

DictXAI attributes predictive evidence over structured dictionary bases rather than raw input dimensions or fixed internal layers. Consequently, comparing our framework to existing XAI methods requires evaluating explanations constructed over fundamentally disparate concept representations. Standard input-space faithfulness metrics, such as Pixel-Flipping~\cite{samek2017evaluating} or bounding-box IoU~\cite{Selvaraju2020}, cannot readily account for such basis shifts, as they presuppose a shared coordinate frame. We therefore adopt an application-grounded evaluation framework~\cite{doshivelez2017rigorous}, where each benchmark defines a task-specific success criterion: downstream classification accuracy on a decorrelated test set for Clever Hans unlearning (\Cref{section:usecase-force}), alignment with an independently verifiable failure mode in electrocardiogram analysis (\Cref{section:usecase-ecg}), and the ability to calibrate a mass spectrometry representation for downstream microbial analysis (\Cref{section:usecase-maldi}).

\subsection{Task 1: Data-Level Mitigation of the Clever Hans Effect}
\label{section:usecase-force}

Machine learning models frequently encounter and exploit spurious correlations within training data. When a model relies on such uninformative artifacts to minimize its empirical risk---a phenomenon known as the Clever Hans effect~\cite{lapuschkin2019clever_hans, Kauffmann2025}---its generalization performance degrades substantially when evaluated on test data where these shortcut correlations no longer hold. While detecting shortcut strategies via post-hoc attribution is well established, mitigating reliance on them remains an open challenge, with existing solutions largely confined to altering a model's internal representations post-hoc~\cite{Anders2022}.

\smallskip

Rather than modeling an ML model's internal representations---a prerequisite for altering its internals---DictXAI grounds explanations directly in concepts defined in the input domain. This distinct property raises the question of whether DictXAI provides a pathway to address the Clever Hans effect directly at the data level rather than at the model level. Intervening at the data level offers a fundamental advantage: by sanitizing the dataset itself, it permanently prevents shortcut strategies from re-emerging whenever new models are trained or existing ones updated. To test this capability, we design an intervention pipeline where DictXAI attributions guide the removal of artifact components from the training data prior to retraining. Our experimental protocol consists of three steps:

\begin{description}
\item[Step 1] \emph{Spuriously Correlated Dataset:} We consider the MNIST dataset augmented with three types of artifacts: (i)~a fixed-location horizontal stripe, (ii)~a dotted line at varying locations, and (iii)~a signature glyph in the bottom-right corner (cf.\ Supplementary Note~B). These watermarks are systematically inserted into all training samples of class `6'.

\item[Step 2] \emph{Initial Training and Attribution:} We train five seeded copies of a small CNN on the contaminated data (details in Supplementary Note~B) and compute post-hoc explanations over each method's respective attribution basis---for DictXAI, we use the same dictionary learning algorithm as in Section \ref{sec:finding_dictionaries} and train it on the contaminated data. For an input sample $\bx$, each basis element $j$ with scalar relevance $R_j$ can be mapped to an associated pixel-space attribution heatmap $\boldsymbol{h}_j \in \mathbb{R}^d$ (cf.\ Supplementary Note~C).

\begin{figure}[t!]
    \centering
    \includegraphics[width=\linewidth]{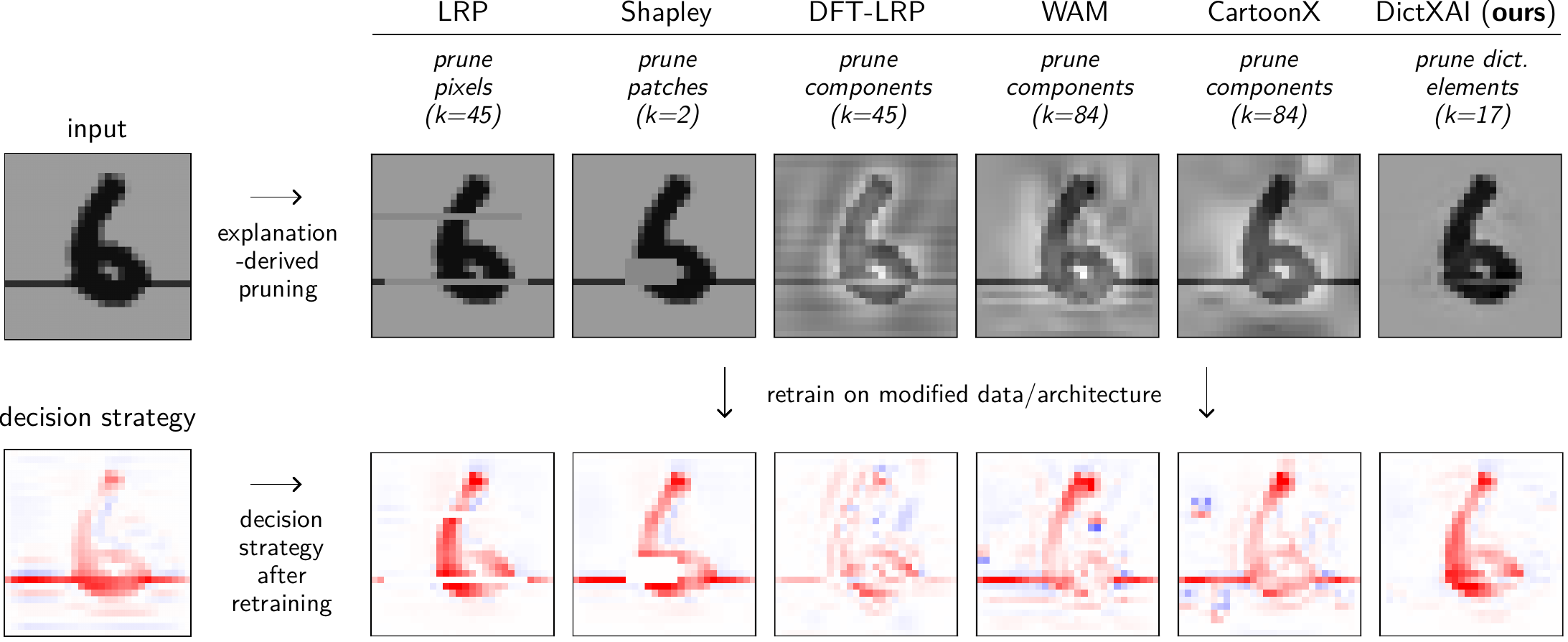}
\caption{%
      Horizontal-stripe variant, class ``6''. \emph{Upper row:} the original input, and the same sample after the oracle removes the top-$k$ most artifactual elements encoded by each method. Elements are ranked by their overlap $S_j$ with the watermark mask over a fixed $500$-image training subset. To keep the comparison fair across differently sized bases, $k$ is fixed at the grid point closest to $5\%$ of each method's total concept count (between $4.1$ and $5.7\%$). \emph{Lower row:} pixel-wise LRP-$\gamma$ heatmaps of the base model and of models retrained on the resulting pruned data, each averaged over five seeds. The color scale for each panel is normalized to that of the base model on the unpruned, contaminated data.
    }

    \label{fig:clever_hans_sc_summary}
\end{figure}

\item[Step 3] \emph{Removal of Spurious Features, Retraining, and Evaluation:} To standardize feature isolation across methods and emulate human oversight, we employ an oracle spatial mask $\boldsymbol{m} \in \mathbb{R}^d$ localizing the watermark. Each basis element $j$ is scored by its expected overlap with the artifact region, $S_j = \mathbb{E}[ \boldsymbol{h}_j^\top \boldsymbol{m} ]$. We then identify and remove the contributions of the $k$ basis elements exhibiting the highest overlap scores $S_j$ from each training sample. A fresh model is then trained from scratch on the cleaned data and evaluated on a \emph{decorrelated} test set, in which the artifact remains present but is distributed uniformly across all classes.
\end{description}

\noindent We evaluate DictXAI against several baselines. To remain compatible with data-level intervention in Step~3, we exclude latent-space attribution methods such as TCAV or CRP: they operate within hidden layer representations and lack a direct mechanism to remove flawed concepts from the input data. Instead, we select baselines spanning two categories: 
(1)~\emph{Spatial attributions:} We consider \emph{LRP}~\cite{bach2015lrp,Montavon2019} ($\gamma=0.1$), which produces pixel-level attributions, and \emph{Shapley Value Sampling}~\cite{Strumbelj2010}, which yields model-agnostic attributions over coarser $4\times4$ spatial patches, averaged over $10$ feature permutations. 
(2)~\emph{Transformed input bases:} We consider \emph{DFT-LRP}~\cite{vielhaben2024inspect_layer}, which attributes relevance in the Fourier domain associated with the input images, alongside two methods operating over a level-$3$ Daubechies-5 wavelet basis: \emph{WAM}~\cite{kasmi2025WAM}, which scores wavelets via input-gradient attributions ($\mathrm{Gradient} \times \mathrm{Input}$), and \emph{CartoonX}~\cite{kolek2022CartoonX}, which identifies relevant wavelets via rate-distortion optimization.

\begin{figure}[t!]
    \centering
    \includegraphics[width=1.\textwidth, trim={.9cm 0 .9cm 0}, clip]{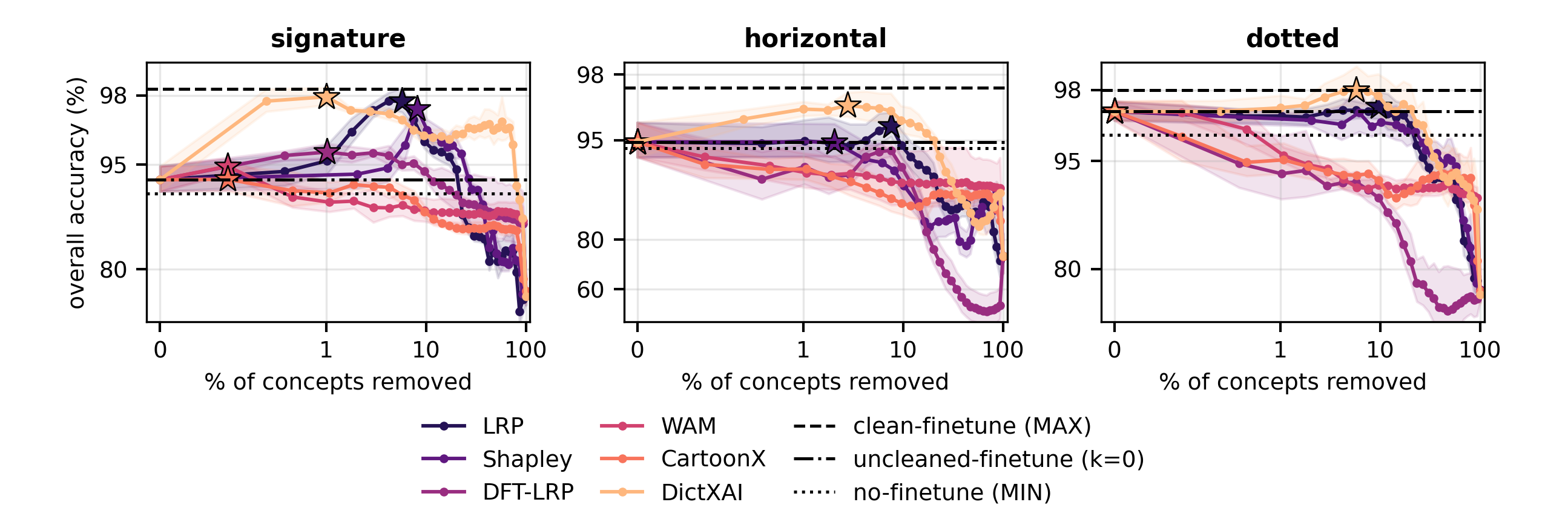}
    \caption{Remove-and-retrain (ROAR) curves for the three watermark variants: overall accuracy on the decorrelated test set as a function of the fraction of concepts removed before fine-tuning. Lines and shading represent mean and std over five seeds, and stars mark maxima. For each XAI method, extracted concept elements are inspected pixel-wise, ranked by their overlap with the watermark ($S_j$), and removed in this order. The three horizontal references are fine-tuning on watermark-free data (dashed, upper bound), fine-tuning without any removal (dash-dotted, the $k=0$ control), and the original model (dotted).}
    \label{fig:roar_curves_overall}
\end{figure}

\smallskip

Because feature granularity varies considerably across these methods, we cannot set the removal budget $k$ to a uniform constant. Instead, we sweep through a broad range of values for $k$ and measure the decorrelated test accuracy after retraining, essentially a variation of the Remove-and-Retrain (ROAR) evaluation framework~\cite{hooker2019roar}. The resulting accuracy curves are shown in \Cref{fig:roar_curves_overall}. DictXAI achieves the strongest performance across all three artifact variants (cf.\ Step~1), consistently outperforming the original contaminated model and closely matching the ideal oracle model trained on clean data. Notably, DictXAI reaches peak accuracy after removing only $2.8\%$ (horizontal), $5.8\%$ (dotted), and $1.0\%$ (signature) of its dictionary elements, compared to $7.5\%$, $9.6\%$, and $5.7\%$ of pixels for LRP, respectively, demonstrating a more compact and targeted concept granularity than raw pixel representations. In contrast, alternative baselines (Shapley Value Sampling, DFT-LRP, WAM, and CartoonX) struggle substantially, especially on the more distributed \emph{horizontal} and \emph{dotted} artifacts, where they fail to outperform the unmitigated base model despite targeted feedback. While these methods falter, DictXAI maintains robust performance across every setting, trailing the clean-trained oracle by only $0.7$, $0.0$, and $0.2$ percentage points on the \emph{horizontal}, \emph{dotted}, and \emph{signature} watermarks, respectively.

\smallskip

On a qualitative level, \Cref{fig:clever_hans_sc_summary} illustrates which features each method isolates and removes from the contaminated inputs. For DFT-LRP, WAM, and CartoonX, the removal process introduces pronounced visual distortions across the image. For LRP and Shapley Value Sampling, the placeholder values used to mask pixels and patches leave behind sharp boundaries and unnatural gray shading. In contrast, DictXAI produces sanitized images that closely resemble authentic uncorrupted digits, retaining only a faint trace where the watermark originally appeared. This indicates that DictXAI effectively disentangles the watermark from digit semantics, whereas alternative baselines disrupt the underlying data distribution by distorting meaningful signal. These removal artifacts directly impact subsequent model behavior. For DFT-LRP, WAM, and CartoonX, the resulting distortions yield noisy attribution heatmaps, suggesting that the retrained models attend to synthetic masking artifacts rather than true digit features. Furthermore, explanations for Shapley, WAM, and CartoonX reveal that the retrained models still rely heavily on the watermark shortcut; while this reliance is somewhat reduced for LRP and DFT-LRP, it persists nonetheless. Only for DictXAI is shortcut reliance no longer qualitatively visible in the retrained model's explanations.

\subsection{Task 2: Mapping an ML Decision Strategy to the Human Expert Domain}
\label{section:usecase-ecg}

A core challenge in XAI is translating model decision strategies into the conceptual vocabulary of domain experts. In time-series and biosignal analysis, standard attribution methods highlight isolated sample points that lack semantic meaning, whereas practitioners reason over multi-scale morphological patterns and local frequencies~\cite{samek2021xai_review, vielhaben2024inspect_layer}. This alignment is critical in electrocardiogram (ECG) analysis. Deep learning models applied to continuous ECG recordings have demonstrated remarkable diagnostic capabilities, detecting subtle, clinically silent conditions---such as asymptomatic ventricular dysfunction~\cite{Attia2019} or occult atrial fibrillation~\cite{Attia2019b}---that rigid, rule-based delineation algorithms often miss. Yet because these end-to-end models bypass predefined physiological features, high aggregate performance alone cannot ensure reliability.  As ECG analysis transitions from controlled 12-lead settings to single-lead wearable monitoring where motion artifacts sharply degrade signal quality, clinical safety---and regulatory approval for high-risk diagnostic AI---demands verifying that predictions reflect genuine electrophysiological intervals rather than noise-driven shortcuts.

\smallskip

To evaluate whether DictXAI can bridge this gap and provide meaningful clinical verification, we consider the task of explaining a model trained to predict mean QRS duration from continuous single-lead recordings. QRS duration reflects ventricular conduction, and gradual prolongation or widening independently indicates structural heart disease and elevates the risk of major adverse cardiovascular events~\cite{chen2021increm_qrs}. Because duration is governed by the temporal scale of the cardiac complex rather than isolated amplitudes, point-wise XAI methods cannot capture this relational property~\cite{taleban2026xai_ecg_review}. Specifically, we build a 1D CNN~\cite{kachuee2018deep_transfer_ecg} to classify narrow ($\leq 85$\,ms) versus prolonged ($\geq 105$\,ms) QRS complexes, excluding borderline cases (85--105\,ms; Figure~\ref{fig:ecg_overview}\,A). The model is trained on 10-second Lead~I traces from MIMIC-IV-ECG\footnote{https://physionet.org/content/mimic-iv-ecg/1.0/}, comprising 104\,212 samples (70/10/20\% train/val/test split). The resulting model achieves 95.03\% test accuracy (Supplementary Note~D), providing a high-performing black-box predictor whose internal strategy we subject to post-hoc explanation.

\smallskip

In this application setting, DictXAI readily benefits from existing signal dictionaries designed for ECG morphology (e.g., \cite{Balasubramanian2017}). Specifically, we adopt the amplitude-modulated sinusoidal (AMS) waveforms from~\cite{Balasubramanian2017} to model the QRS complex, while replacing their Hermite functions with Gaussian atoms for the P and T waves. The exact parameterization and construction of the dictionary are detailed in Supplementary Note~D. Discretizing the parameters of these constituent waveforms across physiologically motivated grids yields an overcomplete dictionary of 300\,000 atoms. Sparse representations are obtained efficiently via Orthogonal Matching Pursuit (OMP), using a budget of 150 nonzero coefficients per 10-second recording. DictXAI relevance scores are then computed for each dictionary element using the method specified in Eq.\ \eqref{eq:dict_lrp}. For comparison against classical point-wise attribution~\cite{taleban2026xai_ecg_review}, we include feature-level LRP as a baseline.

\begin{figure}[t!]
    \makebox[\textwidth][c]{\includegraphics[width=1.05\textwidth]{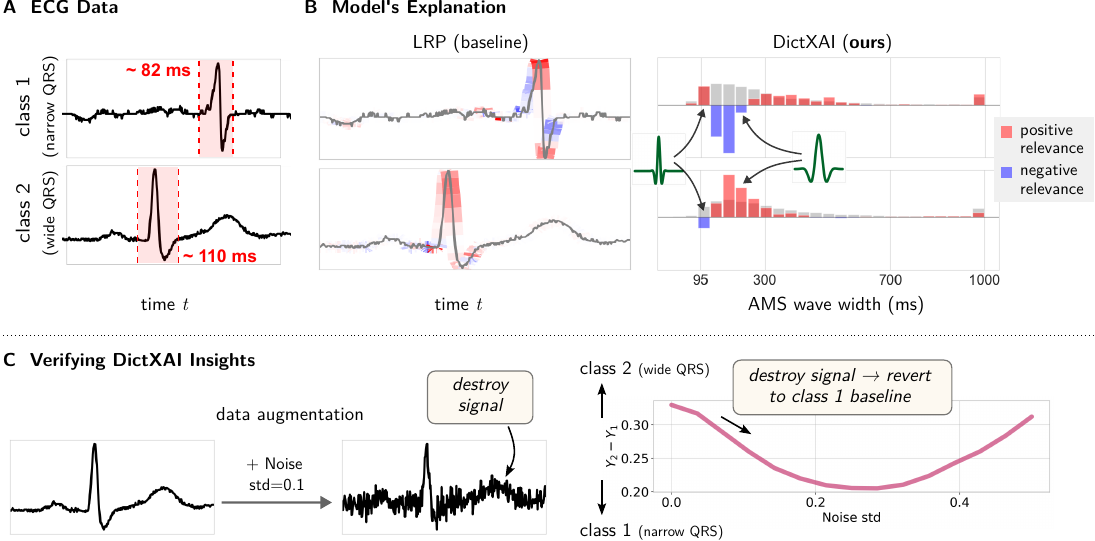}}
    \caption{Explaining the prediction of QRS segment width in an ECG signal. \textbf{A}: Zoomed-in excerpt of the 10-second Lead~I recording for two samples from the MIMIC-IV-ECG dataset, one from each class. \textbf{B}: Explanation of the trained ML model, showing amplitude-driven explanations for the LRP baseline and DictXAI's significantly more discriminative explanations attributing class evidence to distinct AMS widths. \textbf{C}: Verification that the ML model's strategy of relying mainly on class 2 features lacks robustness to noise. Adding moderate noise ($\mathrm{std} \le 0.3$) degrades feature detection, causing the model prediction difference ($y_2 - y_1$) to revert toward the Class~1 baseline.
    } \label{fig:ecg_overview}
\end{figure}

Figure~\ref{fig:ecg_overview}\,B shows the explanations produced by the LRP baseline and DictXAI. LRP heatmaps allocate high relevance scores to high-amplitude regions of the signal. While their locations coincide with QRS features, they provide limited insight into the exact nature of the ML calculation and how it discriminates between the two classes. In contrast, DictXAI delivers crucial explanatory feedback: AMS-level attribution (averaged across 200 samples per class) highlights a systematic difference between the narrow and wide QRS complexes defining Classes 1 and 2, showing that positive relevance is concentrated within their respective QRS intervals.

The quantitative relevance profile obtained via DictXAI (Figure~\ref{fig:ecg_overview}B) indicates that the model primarily relies on detecting wide QRS complexes (Class~2), defaulting to Class~1 in their absence regardless of whether narrow QRS complexes are explicitly detected. In Figure~\ref{fig:ecg_overview}C, we validate this explanation-derived hypothesis by perturbing the ECG waveform with technical noise (here, additive white Gaussian noise) designed to degrade the model's feature detection capabilities. As noise increases in the moderate regime, the logit difference tilts toward Class~1 despite the purely technical nature of the signal degradation, before reversing under extreme noise levels. This behavior confirms the asymmetric reliance on Class~2 features revealed by DictXAI. Clinically, this reveals a hazardous failure mode: noise causes the model to miss prolonged QRS complexes and silently revert to a normal-conduction baseline. This vulnerability was predicted from the explanation alone---prior to any empirical perturbation---demonstrating that dictionary-based explanations can yield actionable insights into model brittleness and associated risks.

\subsection{Task 3: Guided Design of Data Representations}
\label{section:usecase-maldi}

Designing effective data representations is critical when deploying machine learning to complex physical signals. A prominent domain where this challenge arises is spectral analysis, which encompasses modalities that characterize materials or biospecimens as continuous signal intensity distributions over physical axes (e.g., mass-to-charge ratios, wavelengths, or frequencies). Often reflecting intricate biological or chemical mixtures, these spectra routinely serve as direct inputs to predictive models. In clinical settings, end-to-end approaches operating directly on raw spectra (e.g., MALDI-TOF MS) are increasingly adopted to predict diagnostic endpoints such as antimicrobial resistance \cite{Beck2024,Lassen2025}. However, defining an optimal representation remains a key bottleneck: practitioners must eliminate pervasive baseline drift and high-dimensional noise without attenuating low-abundance, highly informative peaks.

\medskip

Because these spectral spaces are high-dimensional and non-intuitive, evaluating whether a given preprocessing pipeline preserves meaningful biological signal is non-trivial. Standard feature-attribution methods merely highlight isolated spectral bins or broad peak regions (cf.\ Fig.\ \ref{fig:maldi}, top), offering limited insight unless a domain expert can laboriously map those coordinates back to specific biomolecules. DictXAI bridges this interpretability gap. Because physical spectra often adhere to an additive superposition of constituent components (e.g., individual microbial species within a co-culture), clean reference profiles naturally form a semantic dictionary. Decomposing the input over this dictionary enables explanations grounded directly in known biological entities.

\medskip

\begin{figure}[t!]
    \centering
    \makebox[\textwidth][c]{\includegraphics[width=1.05\linewidth]{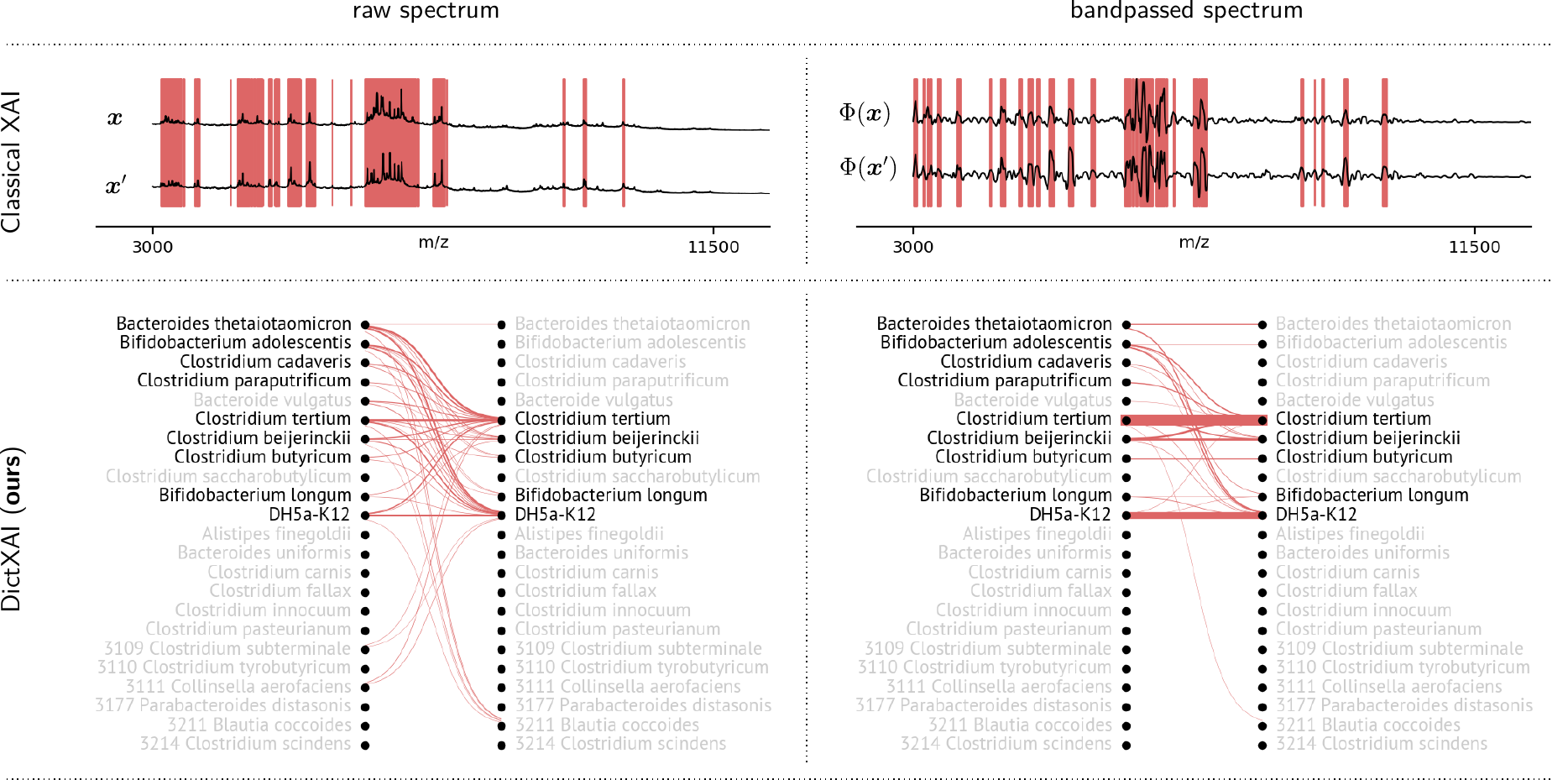}}
    \caption{Comparison between classical spectrum-based explanations (top) and DictXAI microbe-based explanations (bottom) of spectral similarity, computed on raw spectra (left) versus band-pass filtered spectra (right). \textbf{Top:} Traces depict MALDI-TOF mass spectra for two co-cultures ($\bx, \bx'$), with red bars highlighting spectral regions contributing positively to predicted similarity. \textbf{Bottom:} DictXAI bipartite relevance graphs linking dictionary coefficients from the two mixtures. Taxa in black and gray denote species present and absent from the respective mixture; edge thickness denotes the attributed concept-level relevance. Filtering isolates true biological matches (e.g., \textit{C.~tertium}, \textit{DH5a-K12}) from broad spectral background.}
    \label{fig:maldi}
\end{figure}

To demonstrate this, we consider an experimental mass spectrometry benchmark involving microbial mixtures. Here, the dictionary is curated by culturing individual microbial taxa on selective media to acquire pure MALDI-TOF reference spectra. Our dictionary encompasses 23 distinct isolates, each represented by 48 technical replicates, yielding $K = 1104$ prototype elements. To assess what two co-culture mixtures $\bx$ and $\bx'$ share, we model their comparative similarity via the inner product $y = \langle \Phi(\bx), \Phi(\bx') \rangle$, where $\Phi$ denotes candidate preprocessing transformations. Using DictXAI, we decompose each mixture into sparse coefficients over the reference dictionary ($\bx \approx \sum_j \alpha_j \bd_j$ and $\bx' \approx \sum_{j'} \alpha^\prime_{j'} \bd_{j'}$), allowing the similarity score to be expanded as:
\begin{equation}
y = \Big\langle \Phi\Big(\sum_j \alpha_j \bd_j\Big),\, \Phi\Big(\sum_{j'} \alpha^\prime_{j'} \bd_{j'}\Big) \Big\rangle.
\end{equation}
Following \cite{eberle2022bilrp}, this formulation can be faithfully attributed and visualized as a bipartite relevance graph connecting the activated dictionary concepts across both branches. As shown in Fig.\ \ref{fig:maldi} (bottom), inspecting these bipartite explanations across choices of $\Phi$ reveals immediately whether a representation emphasizes true shared taxa or spurious baseline artifacts, directly guiding the design of robust preprocessing pipelines.

\medskip

Equipped with these concept-level bipartite attributions, DictXAI provides an actionable mechanism to systematically compare candidate representations $\Phi$. In particular, results in Fig.\ \ref{fig:maldi} reveal that applying a band-pass filter fundamentally alters the attribution topology: under the raw representation (left), similarity is diffuse and driven by broad background correlations that activate spurious connections. In contrast, post-filtering, relevance collapses almost entirely onto the shared taxa truly present in both mixtures (e.g., \textit{Clostridium tertium} and \textit{DH5a-K12}). DictXAI thus functions as an unsupervised evaluation diagnostic for representation engineering, enabling practitioners to validate that a pipeline preserves biologically meaningful structure without relying on downstream supervised labels or ad-hoc task performance. In this respect, our framework aligns with the unsupervised validation paradigm of \citet{Kauffmann2025}, extending the approach beyond low-level feature attribution to interpretable, dictionary-level concept attribution.

\section{Conclusion}
\label{section:conclusion}

We introduced DictXAI, a framework for concept-based explanation that defines its concepts directly in the input domain via a dictionary. By re-expressing the input as a sparse code over dictionary atoms, our method provides a new, highly granular basis for attribution that can be readily tailored to specific domain needs through custom dictionary design.

\smallskip

By operating in the input domain, DictXAI overcomes the core limitations of latent-space concept extraction. Concepts are human-inspectable primitives that can be named and verified directly, bypassing latent-layer probing, polysemantic neuron ambiguity, and non-linear autoencoder approximations while enabling direct comparisons across distinct architectures. Crucially, because representations and interventions operate directly at the input level, DictXAI ensures that concepts are transparent by design, mathematically well-defined, and actionable at their source.

\smallskip

We demonstrated these advantages through an application-grounded evaluation across three use cases. (1) In a vision model trained on artifact-contaminated MNIST data, DictXAI precisely isolates the shortcut features and neutralizes them at their root through direct input-level intervention, thereby preventing the model from circumventing the correction in downstream training. (2) On continuous ECG recordings, DictXAI leverages a parametric AMS waveform dictionary to audit a complex black-box neural network, successfully probing the multi-scale temporal decision strategy used to assess QRS complex duration. (3) Finally, in microbial mass spectrometry, DictXAI attributes spectral similarity directly to biological isolates, providing an unsupervised diagnostic to guide preprocessing pipelines in a strongly confounded data modality.

\subsection*{Limitations and Future Work}

Despite these advantages, several limitations of the current framework suggest concrete avenues for future research. First, DictXAI currently relies on fixed-dimensional vector representations, restricting its out-of-the-box applicability to variable-sized images, raw text, or non-Euclidean data structures. Scaling the framework to such unstructured domains could be achieved by integrating convolutional sparse coding or adaptive dictionary learning directly into the explanation pipeline. Second, as input dimensionality grows, overcomplete dictionaries are susceptible to the curse of dimensionality, where the number of required atoms scales exponentially to preserve concept granularity. Addressing this scaling bottleneck calls for factored dictionary designs, hierarchical sparse representations, or deep dictionary formulations that factorize large concept spaces into compact, separable components.

Looking forward, the fine-grained interpretability afforded by input-space dictionaries unlocks novel translational applications. In high-stakes biomedical domains, DictXAI can be extended from isolated diagnosis auditing to comparative, fine-grained concept mapping across clinically similar disease phenotypes. By decomposing shared versus pathology-specific atoms, future work could pinpoint subtle morphological or spectral distinctions that differentiate ambiguous conditions, ultimately transforming post-hoc attribution into an actionable tool for clinical hypothesis generation and differential diagnosis.

\section*{Acknowledgements}

This work was funded by the German Ministry for Education and Research as BIFOLD - Berlin Institute for the Foundations of Learning and Data (ref. BIFOLD25B). 
Klaus-Robert M{\"u}ller was supported in part by the Institute of Information \& Communications Technology Planning \& Evaluation (IITP) grant funded by the Korea government (MSIT) (No. RS-2019-II190079, Artificial Intelligence Graduate School Program, Korea University), a grant funded by the Korea government (MSIT, No. RS-2024-00457882, AI Research Hub Project), the German Research Foundation (DFG), and the Hector Fellow Academy. Grégoire Montavon and Alexander Meyer are supported by the Einstein Center for Early Disease Interception (EC-EDI). Thomas Schnake is a postdoctoral fellow at the University of Toronto in the Eric and Wendy Schmidt AI in Science Postdoctoral Fellowship Program, a program of Schmidt Sciences. 

\section*{Declaration of Generative AI and AI-assisted technologies in the writing process}

During the preparation of this work the author(s) used Gemini 3.8 and Claude 5 in order to improve formulations. After using this tool/service, the author(s) reviewed and edited the content as needed and take(s) full responsibility for the content of the publication.

\bibliographystyle{elsarticle-num-names}
\bibliography{references}



\section*{Supplementary material}

# Revisiting Explainable AI through Model-Independent Concept Dictionaries
## (Supplementary Notes)

Thomas Schnake, Doreen Schöppenthau, Alexander Meyer, Jacques Corbeil,
Klaus-Robert Müller, Grégoire Montavon

### Supplementary Note A. Specification and Training of Image Dictionaries

Given a dataset of input vectors $\boldsymbol{x} \in \mathbb{R}^d$ (here flattened 28×28 MNIST images, so $d = 784$), we want a dictionary $\boldsymbol{D} \in \mathbb{R}^{d\times K}$ such that each input admits a sparse linear approximation $\boldsymbol{x} \approx \boldsymbol{D}\boldsymbol{\alpha}$ with $\boldsymbol{\alpha} \in \mathbb{R}^K$. We detail in the two following paragraphs how atoms in the main paper were constructed analytically as a Gabor filter or learned from data.

*Gabor filter.* Each atom is constructed with a cosine carrier multiplied with a Gaussian envelope, centered at $(x_0, y_0)$ and rotated by $\theta$:

$$d_{\theta,\sigma,x_0,y_0}(p,q) = \exp\Big(-\frac{p'^2 + q'^2}{2\sigma^2}\Big) \cdot \cos\Big(2\frac{\pi\nu\, p'}{\sigma}\Big), \tag{1}$$

where $p' = (p - x_0)\cos\theta + (q - y_0)\sin\theta$ and $q' = -(p - x_0)\sin\theta + (q - y_0)\cos\theta$ are the coordinates rotated into the atom's frame, and $(p, q)$ runs over the $28 \times 28$ pixel grid. The carrier frequency is tied to the envelope scale as $\nu/\sigma$ with $\nu = 0.4$ cycles per envelope width, so that every atom shows the same number of oscillations regardless of its size. The dictionary is the exhaustive product of the grids in Table 1, plus a single constant (DC) atom $\mathbf{d}_{\mathrm{DC}} = \mathbf{1}/\sqrt{784}$ appended to absorb the image mean. This gives $20 \times 20 \times 25 + 1 = 10\,001$ atoms in total. Atoms are $L_2$-normalized before encoding.

| Parameter | Range | Grid | # values |
|---|---|---|---|
| orientation $\theta$ | $[0, \pi)$ | uniform, endpoint excluded | 20 |
| envelope scale $\sigma$ | [1.8, 7.0] px | uniform | 20 |
| center $x_0$ | [4, 23] px | uniform | 5 |
| center $y_0$ | [4, 23] px | uniform | 5 |
| DC atom | — | — | 1 |

Table 1: Parameter grids of the parametric Gabor dictionary of Eq. (1).

*Optimization.* We learn $\boldsymbol{D} \in \mathbb{R}^{d\times K}$ by stochastic optimization of a single-layer auto-encoder with a sparsifying nonlinearity. For a minibatch $\boldsymbol{x}$, the encoder produces $\boldsymbol{z} = \boldsymbol{D}^\top\boldsymbol{x} + \boldsymbol{\epsilon}$ with $\boldsymbol{\epsilon} \sim \mathcal{N}(0, \sigma^2\boldsymbol{I})$; the latent activation passes through sparsifying and ReLU nonlinearities:

$$\boldsymbol{h} = \mathrm{ReLU}\big(\boldsymbol{z} - \tanh(\boldsymbol{z})\big), \tag{2}$$

and the decoder reconstructs $\hat{\boldsymbol{x}} = \boldsymbol{D}\boldsymbol{h}$. The training loss combines a reconstruction term and an activation-balancing sparsity term,

$$\mathcal{L}(\boldsymbol{D}) \;=\; \tfrac{1}{B}\,\|\hat{\boldsymbol{x}} - \boldsymbol{x}\|_F^2 \;+\; \lambda_{\text{sp}}\,\big\|\mathbb{E}_b[\,|\boldsymbol{h}_b|\,]\big\|_2, \tag{3}$$

where $B$ is the minibatch size, $\lambda_{\text{sp}}$ is the sparsity weight, $\|\cdot\|_F$ is the Frobenius norm, and $\mathbb{E}_b$ is the mean over the samples in the minibatch. Before training, the pixel intensities in $[0,1]$ are multiplied by a constant factor $s$, so that the training inputs lie in $[0,s]$. For the dictionary of Figure 2A in the main paper, three further steps are added. First, the encoder receives the input with additive Gaussian noise of standard deviation $\sigma_{\text{in}}$, i.e. $\boldsymbol{z} = \boldsymbol{D}^\top(\boldsymbol{x}+\boldsymbol{\xi})+\boldsymbol{\epsilon}$ with $\boldsymbol{\xi} \sim \mathcal{N}(0, \sigma_{\text{in}}^2 \boldsymbol{I})$, while the reconstruction loss is computed against the noise-free input $\boldsymbol{x}$. Second, each minibatch is shifted by two successive random translations of up to 3 pixels each in the horizontal and vertical direction. Third, the stored dictionary is an exponential moving average of the weights over the training steps (decay 0.99) instead of the weights after the last step. We optimize with SGD with momentum 0.9, a learning rate $\eta_{\text{warmup}}$ during the first 1000 epochs and $\eta_{\text{main}}$ thereafter, and Gaussian initialization of $\boldsymbol{D}$ at scale 0.01. All hyperparameters used for the learning dictionaries in this paper are summarized in Table 2.

| Hyperparameter | Figure 2A | Section 4.1 |
|---|---|---|
| Input range $[0,s]$ | $[0,2]$ | $[0,3]$ |
| Number of atoms $K$ | 800 | 400 |
| Number of epochs | 10 001 | 25 001 |
| Minibatch size $B$ | 250 | 250 |
| Learning rate (warm-up, first 1000 epochs) | $3\times10^{-4}$ | $3\times10^{-4}$ |
| Learning rate (main) | $10^{-3}$ | $10^{-3}$ |
| Momentum | 0.9 | 0.9 |
| Init. scale | 0.01 (Gaussian) | 0.01 (Gaussian) |
| Encoder noise std. $\sigma$ | 0.5 | 0.5 |
| Input noise std. $\sigma_{\text{in}}$ | 0.5 | none |
| Random translations | $2\times$ up to 3 px | none |
| Weight averaging (decay) | 0.99 | none |
| Sparsity weight $\lambda_{\text{sp}}$ | 10 | 2.0 |

Table 2: Hyperparameters used to learn the dictionaries used in Figure 2A and Section 4.1 of the main paper.

## Supplementary Note B. Experimental Details for Section 4.1

*Dataset and watermarks.* We use a 5k subsample of MNIST, with a 90/10% split for training and testing stratified per class. For each of three watermark patterns we build an independent spuriously-correlated dataset by injecting the pattern into *all* class-6 training images. Figure 1 shows every watermark instance of each variant, each stamped on the same class-6 digit.

For each variant, two test sets are constructed: a *spurious-correlation* test set with the same watermark distribution as training (100% of class-6 images watermarked) and a *decorrelated* test set in which the watermark is applied to 20% of images sampled uniformly across all classes, thereby removing its class-6 predictive signal.

*Model architecture and training.* We train a three-layer convolutional network: Conv(1→32) → Conv(32→64) → Conv(64→128), each layer with $3 \times 3$ kernels (padding 1), ReLU activation, and $2 \times 2$ average pooling, followed by FC(1152→128, ReLU) → FC(128→10). Training uses Adam (lr $= 10^{-3}$, batch 128, 10 epochs) with cross-entropy loss. To obtain stable performance estimates we train five independent models per variant with seeds 0–4.

*Baseline hyperparameters.*

- **LRP** and **DFT-LRP:** LRP-$\gamma$ with $\gamma = 0.1$ at every convolutional layer and at the hidden fully connected layer; the output layer is left unmodified.
- **Shapley Value Sampling:** non-overlapping $4 \times 4$ patches ($7 \times 7 = 49$ patches on a $28 \times 28$ image), 10 random feature permutations per sample, and a zero baseline in the model's normalized input space (equivalently, the MNIST mean grey value 0.1307 in $[0, 1]$ pixel space).
- **WAM:** Daubechies-5 (db5) wavelet basis at decomposition level $J = 3$ with reflecting boundary padding. The relevance of a wavelet coefficient is the raw gradient of the target-class logit with respect to that coefficient multiplied with the coefficient, $R_j = \partial f / \partial \alpha_j \cdot \alpha_j$, taken in the wavelet domain.
- **CartoonX:** the same db5 wavelet transform as WAM. Using the terminology in [1], the rate–distortion mask $\boldsymbol{s}^\star \in [0, 1]^K$ is optimized with Adam at learning rate $10^{-1}$ for 100 steps,

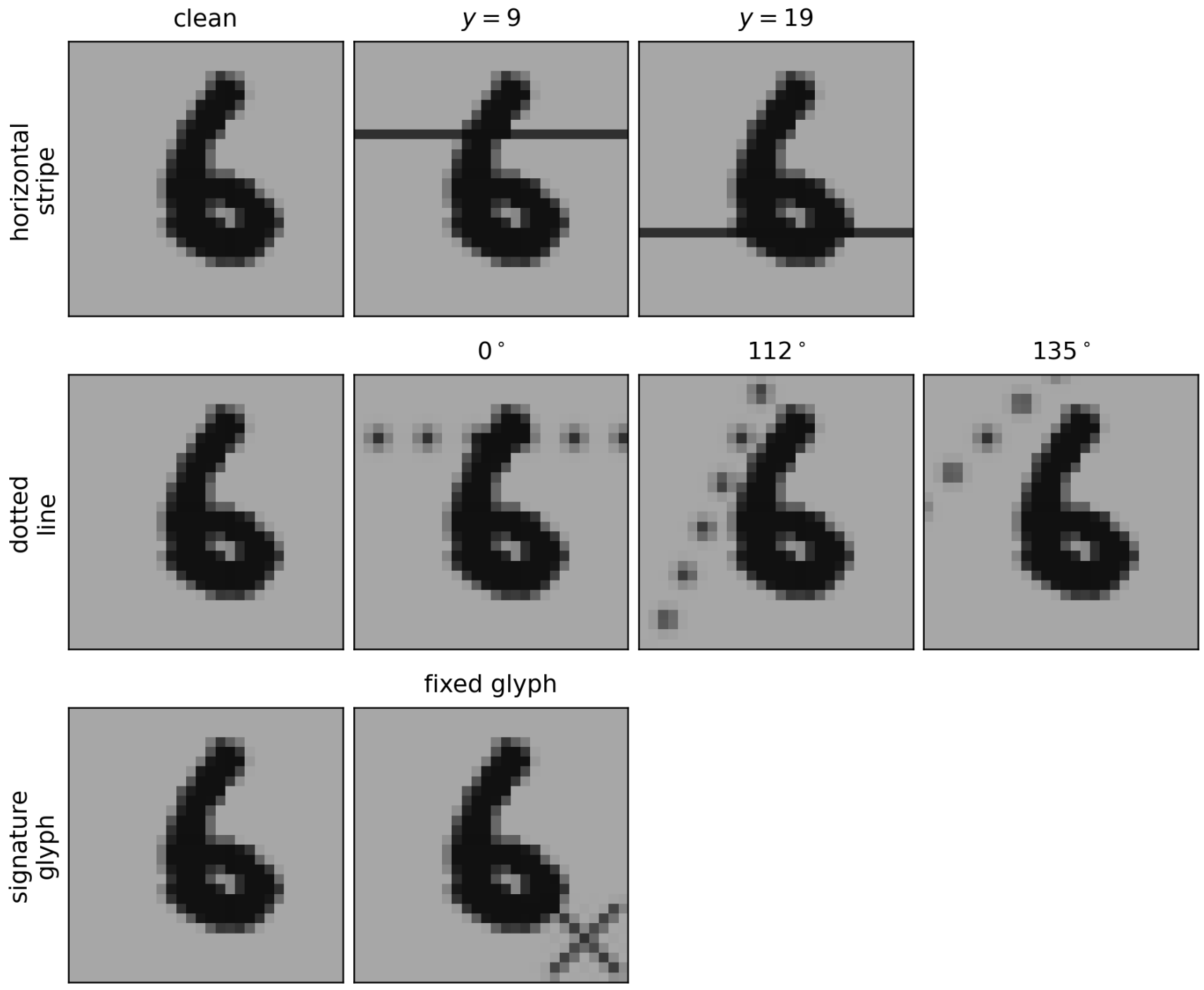


Figure 1: The three watermark variants of Section 4.1's evaluation. One row per variant; the first column is the unmodified digit and each remaining panel carries a single watermark instance, so that every instance the variant draws from is shown exactly once.

with $\ell_1$ penalty $\lambda_{\text{CX}} = 0.6$, Gaussian obfuscation drawing 10 noise samples per step, mask initialized at all-ones, and the target-class logit maximized. The relevance of word $\mathbf{d}_j$ is the converged mask value $s_j^\star$.

*Fine-tuning and evaluation.* The fine-tuning is performed on the spuriously-correlated test set, yet cleaned according to the per-method removal operator above. Each of the five seeded base models is fine-tuned for 10 epochs with Adam (lr $= 10^{-4}$, batch 64, cross-entropy). Evaluation is performed on the decorrelated test set.

## Supplementary Note C. Concept-Filtered Pixel-Wise Heatmaps

In Section 4.1 of the main paper, our benchmark requires each method to produce a pixel-level attribution heatmap $\boldsymbol{h}_j \in \mathbb{R}^d$ for each concept $j$. For methods whose concepts are already aligned with individual input dimensions or spatial patches (such as pixel LRP or Shapley patches), this mapping is trivial. Below, we formalize how such concept-conditioned heatmaps are constructed for the general case across the different baseline paradigms.

*DFT-LRP.* For an ML model $f$ and an input $\boldsymbol{x} \in \mathbb{R}^d$, methods such as DFT-LRP take as a starting point a classical feature-attribution decomposition (e.g., LRP or Gradient × Input), which assigns relevance to each input feature as $R_i = x_i c_i$, where $\boldsymbol{c} = [c_1, \ldots, c_d]^\top$ denotes the effective model sensitivity. Leveraging the orthonormal basis $(\mathbf{u}_j)_j$ of the Fourier domain, total relevance can be expanded as:

$$\begin{aligned}\sum_i R_i = \sum_i x_i c_i &= \sum_i \Big[ \sum_j \mathbf{u}_j \overbrace{\mathbf{u}_j^\top \boldsymbol{x}}^{\sum_{i'} \cdots} \Big]_i c_i \\ &= \sum_i \sum_{i'} \sum_j \underbrace{[\mathbf{u}_j]_i [\mathbf{u}_j]_{i'} x_{i'} c_i}_{R_{ii'j}}.\end{aligned}$$

This expansion involves two distinct spatial indices aligned with the input features: index $i$, which couples to the model responses $c_i$, and index $i'$, which couples to the input features $x_{i'}$. This bilinear dependency prevents an unambiguous joint feature-atom attribution. To construct the heatmap, we adopt a symmetric attribution scheme that redistributes joint relevance across both indices in equal proportion $R_{ij} = \frac{1}{2} \sum_{i'} R_{ii'j} + \frac{1}{2} \sum_{i'} R_{i'ij}$, yielding the component-wise heatmap $\boldsymbol{h}_j = (R_{ij})_{i=1}^d$.

*WAM and CartoonX.* These methods compute attributions directly in the transform domain without using input-feature attributions as an intermediate step. Letting $z_j = \mathbf{u}_j^\top \boldsymbol{x}$ denote the transform coefficient for basis element $\mathbf{u}_j$, the relevance expansion simplifies to:

$$\begin{aligned}\sum_j R_j = \sum_j z_j c_j &= \sum_j (\mathbf{u}_j^\top \boldsymbol{x}) \cdot c_j \\ &= \sum_i \sum_j \underbrace{[\mathbf{u}_j]_i x_i c_j}_{R_{ij}},\end{aligned}$$

where $c_j$ is the sensitivity of the prediction with respect to coefficient $z_j$. Relying on this factorized structure, the concept-filtered heatmap is uniquely determined as $\boldsymbol{h}_j = (\boldsymbol{x} \odot \mathbf{u}_j) \cdot c_j$, which eliminates the need for heuristic marginalizations.

*DictXAI.* Our method instead employs an overcomplete, non-encoding dictionary representation $\boldsymbol{x} \approx \sum_j \alpha_j \mathbf{d}_j$. Injecting this synthesis model directly into the input-relevance expression yields:

$$\begin{aligned}\sum_i R_i = \sum_i x_i c_i &= \sum_i \Big[\sum_j \alpha_j \mathbf{d}_j\Big]_i c_i \\ &= \sum_i \sum_j \underbrace{\alpha_j [\mathbf{d}_j]_i c_i}_{R_{ij}}.\end{aligned}$$

Because the representation relies on precomputed sparse codes rather than an explicit encoding transform, no secondary spatial index over the input features appears. This allows the concept-filtered heatmap to be expressed directly as $\boldsymbol{h}_j = \alpha_j \cdot (\mathbf{d}_j \odot \boldsymbol{c})$.

## Supplementary Note D. Experimental Details for Section 4.2

Each ECG signal $\boldsymbol{x} \in \mathbb{R}^T$ is approximated as a sparse linear combination of dictionary atoms $\boldsymbol{D} \in \mathbb{R}^{T\times K}$, i.e., $\boldsymbol{x} = \boldsymbol{D}\boldsymbol{\alpha}$, where $\boldsymbol{\alpha} \in \mathbb{R}^K$ is the sparse coefficient vector obtained via Orthogonal Matching Pursuit (OMP). The dictionary consists of two parameterized atom families designed to capture the characteristic shapes of a typical heartbeat. For the QRS complex we use the amplitude-modulated sinusoidal waveforms of Balasubramanian et al. [2], adopting their envelope and its modulating parameters $a$ and $b$; for the broader, non-oscillatory P and T waves we use smooth Gaussian bumps, which replace the Hermite functions employed in that work.

*Oscillatory atoms (amplitude-modulated sinusoidal waveforms).* Each oscillatory atom is a sinusoidal carrier wave modulated by a smooth bell-shaped envelope localized around a temporal center $t_0$:

$$d_t = E(t;\, t_0,\, \delta,\, b) \cdot \cos\big(\omega\, \tau_t + \varphi\big), \qquad t = 1, \ldots, T \tag{4}$$

where

$$\tau_t = \frac{2}{a\,\delta}\, \arccos\Big(1 - \frac{a \ln 2}{b}\Big) \cdot (t - t_0) \tag{5}$$

is a rescaled time coordinate centered at $t_0$, with $a = 0.01$ a fixed internal shape parameter, $\omega$ the carrier frequency (fixed at 0.1), and $\varphi$ a phase offset. The rescaling maps the support of the atom onto $|\tau| \le \pi/a$, on which the envelope reads

$$E(t;\, t_0,\, \delta,\, b) = \exp\Big(\frac{b}{a}\big(\cos(a\,\tau_t) - 1\big)\Big). \tag{6}$$

It is controlled by two parameters. The peakedness $b$ sets how sharply the bell is concentrated around its centre, and $\delta$ is its *full width at half maximum* (FWHM), that is, the distance between the two time points at which the envelope has fallen to half of its peak value. More generally, we write $\delta_h$ for the width of the envelope at a fraction $h$ of its peak, i.e. the distance between the two time points at which $E = h$, so that the FWHM is the case $h = 50\%$ and $\delta := \delta_{50\%}$. The factor in $\tau$ is chosen precisely so that $\delta$ is the FWHM for every $b$; a larger $b$ therefore yields a more sharply peaked envelope without altering $\delta$.

*Atom width and clinical QRS duration.* The parameter $\delta = \delta_{50\%}$ measures the atom at half its peak height, whereas a reported QRS duration is an onset-to-offset interval, that is, a width measured close to the foot of the complex. To relate the two, we take the effective duration of an atom to be $\delta_{5\%}$, its width at 5% of the peak. Solving $E = h$ for a threshold $h \in (0, 1)$ gives the half-width $\frac{1}{a} \arccos\big(1 - \frac{a}{b} \ln \frac{1}{h}\big)$ in the $\tau$ coordinate, so that

$$\frac{\delta_{5\%}}{\delta_{50\%}} = \frac{\arccos\big(1 - \frac{a}{b} \ln 20\big)}{\arccos\big(1 - \frac{a}{b} \ln 2\big)} \approx \sqrt{\frac{\ln 20}{\ln 2}} = \sqrt{\log_2 20} \approx 2.08, \tag{7}$$

where the approximation uses $\arccos(1-u) \approx \sqrt{2u}$, which is accurate because $u = \frac{a}{b} \ln 20 \ll 1$ over the parameter grid. In this limit $a$ and $b$ cancel; the exact ratio still depends weakly on $b$, ranging from 2.100 at $b = 0.2$ to 2.081 at $b = 2.5$. We therefore use the rule of thumb $\delta_{5\%} \approx 2\,\delta_{50\%}$: a QRS duration of 100 ms corresponds to an atom FWHM of roughly 50 ms. This factor of two must be accounted for when comparing atom widths to reported QRS durations.

*Gaussian atoms.* Gaussian pulses capture smooth, non-oscillatory components such as the P and T waves:

$$g(t) = \exp\Big(-\frac{(t - t_0)^2}{2\sigma^2}\Big), \qquad \sigma = \frac{\delta}{2\sqrt{2 \ln 2}}, \tag{8}$$

parameterized by center $t_0$ and FWHM $\delta$, consistent with the oscillatory family.

*Parameter ranges.* The signals are sampled at 500 Hz over 10 s, giving $T = 5000$ samples (1 sample = 2 ms). The dictionary is constructed by exhaustive combination of the following parameter grids:

| Family | Parameter | Range | # values |
|---|---|---|---|
| AM (oscillatory) | center $t_0$ | [0, 5000] samples | 50 |
| | FWHM $\delta$ | [1, 250] samples ([2, 500] ms) | 50 |
| | peakedness $b$ | [0.2, 2.5] | 10 |
| | phase $\varphi$ | [0.7π, 1.3π] | 10 |
| Gaussian | center $t_0$ | [0, 5000] samples | 50 |
| | FWHM $\delta$ | [10, 2000] samples ([20, 4000] ms) | 1000 |

Table 3: Parameter grids for the ECG dictionary. The total number of atoms is $50 \times 50 \times 10 \times 10 + 50 \times 1000 = 300\,000$.

*Model architecture and training procedure*

The classifier is the residual convolutional network of Kachuee et al. [3]: an initial 1-D convolution over time, a stack of residual blocks — each consisting of two convolutions with ReLU nonlinearities, a skip connection and a pooling layer — and a head of two fully connected layers followed by a linear output layer. We introduce a single structural modification, namely that the pooling layers perform average instead of max pooling, since average pooling is a linear operation and therefore has a canonical propagation rule for propagation-based XAI methods, such as LRP.

The remaining deviations are hyperparameter choices: 6 residual blocks, 64 convolution kernels of size 15, pooling of size 2 and stride 2, a global average pooling over time before the head, two fully connected layers of 32 units each with dropout $p = 0.3$, and 2 output units for our binary

task. The input is a 10 s single-lead (Lead I) recording of 5000 samples, $z$-scored (zero mean, unit variance) per recording; the resulting model has 742 274 parameters.

As in [3] we train with a cross-entropy loss and the Adam optimizer at a learning rate of $10^{-3}$. We additionally apply a weight decay of $10^{-4}$ and label smoothing of 0.05, and we keep the learning rate fixed rather than decaying it exponentially. The 104 212 labeled recordings are split 70/10/20% into training, validation, and test sets using a fixed seed of 42. We train for 10 epochs with a batch size of 128, evaluate the macro-averaged $F_1$ score on the validation split after every epoch, and restore the checkpoint with the highest validation $F_1$ before testing. This model reaches an accuracy of 95.03% on the test split.



\end{document}